\documentclass{article} 
\usepackage{nips14submit_e,times}
\usepackage{bm}

\usepackage{hyperref}
\usepackage{url}

\usepackage{amsmath, amssymb, amsthm}
\usepackage[linesnumbered,ruled]{algorithm2e}

\usepackage{xcolor,graphicx}

\DeclareMathOperator{\diag}{diag}
\DeclareMathOperator{\image}{Im}
\DeclareMathOperator{\circulant}{Circ}

\DeclareMathOperator{\trace}{Tr}

\newcommand{\vct}[1]{\bm{#1}}
\newcommand{\mtx}[1]{\bm{#1}}

\newtheorem{definition}{Definition}
\newtheorem{lemma}{Lemma}

\newtheorem{theorem}{Theorem}

\newtheorem{corollary}{Corollary}

\title{What Happens on the Edge, Stays on the Edge: \\ Toward Compressive Deep Learning}

\author{Yang Li and Thomas Strohmer  \\
Department of Mathematics \\
University of California, Davis \\
Davis, CA 95616, USA \\
\texttt{\{ly,strohmer\}@math.ucdavis.edu} \\
}

\nipsfinalcopy 

\begin{document}

\maketitle

\begin{abstract}
Machine learning at the edge offers great benefits such as  increased privacy and security, low latency, and more autonomy. However, a major challenge is that
many devices, in particular edge devices, have very limited memory, weak processors, and scarce energy supply. We propose a hybrid hardware-software framework that has the potential to significantly reduce the computational complexity and memory requirements of on-device machine learning.  In the first step, inspired by compressive sensing, data is collected in compressed form simultaneously with the sensing process. Thus this compression happens already at the hardware level during data acquisition. But unlike in compressive sensing, this compression is achieved via a  projection operator that is specifically tailored to the desired machine learning task. The second step consists of a specially designed and trained deep network. As concrete example we consider the task  of image classification, although the proposed framework is more widely applicable. An additional benefit of our approach is that it can be easily combined with existing on-device techniques. Numerical simulations illustrate the viability of our method.
\end{abstract}

\section{Introduction}

\subsection{Machine Learning on Constrained Devices}


While the leitmotif  {\em ``local sensing and remote inference''}\footnote{We borrowed this slogan
from~\cite{Deisher}}  has been at the foundation of a range of successful AI applications, there exist many
scenarios in which it is highly preferable to run machine learning algorithms either directly on the device or at least locally on the edge instead of remotely. As tracking and selling our digital data has become a booming business model~\cite{zuboff2019}, there is an increasing and urgent need to preserve privacy in digital devices.
Privacy and security are much harder to compromise if data is processed locally instead of being sent to the cloud. Also, some AI applications may be deployed ``off the grid'', in regions far away from any mobile or internet coverage. In addition, the cost of communication may become prohibitive and scaling to millions of devices  may be impractical~\cite{Deisher}. Moreover, some applications cannot afford latency issues that might result from remote processing. 

Hence, there is a strong incentive for local inference ``at the edge'', {\em if} it can be accomplished with little sacrifice in 
accuracy and within the resource constraints  of the edge device. Yet, the difficulty in ``pushing AI to the edge''  lies exactly in these resource constraints of many edge devices regarding computing power, energy supply, and memory.

In this paper we focus on  such {\em constrained devices} and propose a hybrid hardware-software framework that has the potential to significantly reduce the computational complexity and memory requirements of on-device machine learning.

\subsection{Outline of our approach}

When we deploy an AI-equipped device in practice, we know a priori what kind of task this device is supposed to carry out.
The key idea of our approach can thus be summarized as follows: We can take this knowledge into account already in the data acquisition step itself and try to measure only the task-relevant information, thereby  we significantly lower the size of the data that enter the device and thus reduce the computational burden and memory requirements of that device.

To that end we propose {\em compressive deep learning}, a hybrid hardware-software approach that can be summarized by two steps, see also Figure~\ref{fig:compdl}. First, we construct a projection operator that is specifically tailored to the desired machine learning task and which is determined by the entire training set (or a subset of the training set). This projection operator compresses the data simultaneously with the sensing process, like in standard compressive sensing~\cite{FR13}. But unlike compressive sensing, our projection operator is tailored specifically to the intended machine learning task, which therefore allows for a much more ``aggressive'' compression.
The construction of the projection operator is of course critical and various techniques are described in Section~\ref{s:math}.
This projection will be implemented in hardware, thus the data enter the software layer of the device already in compressed form. The data acquisition/compression step is followed by a deep network that processes the compressed data and carries out the intended task.

\begin{figure}[ht]
\centering
\includegraphics[width=0.95\textwidth]{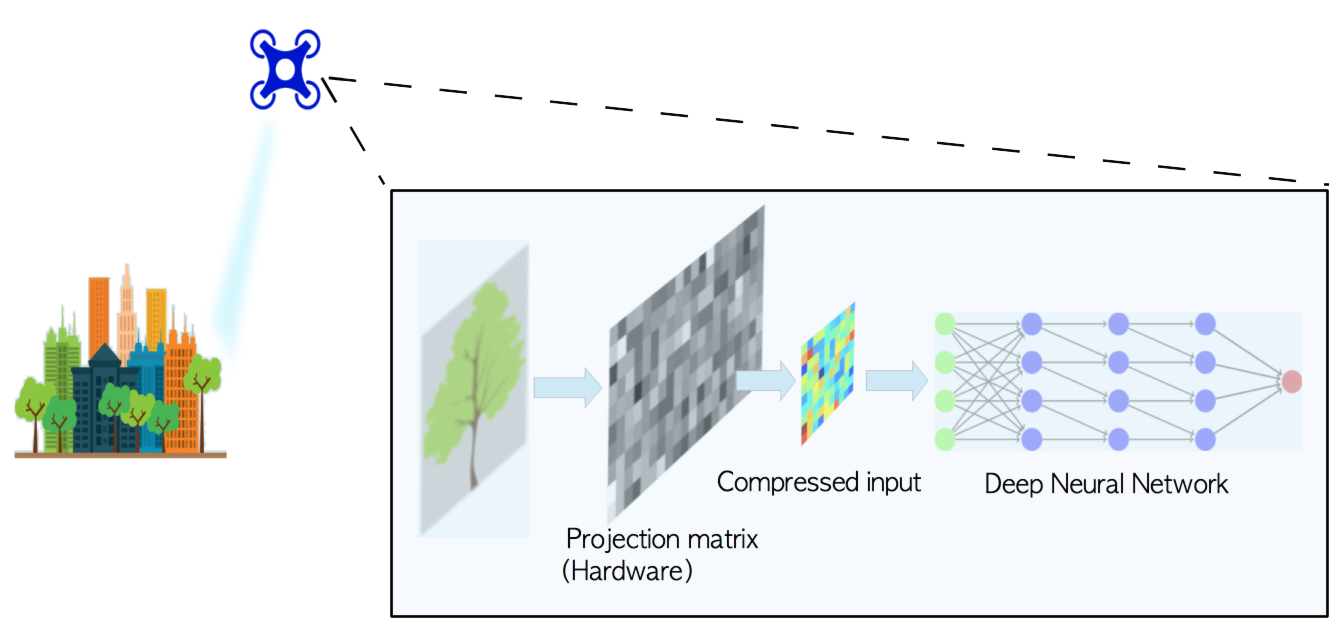}
\caption{Schematic depiction of compressive deep learning: Data acquisition and compression are carried out simultaneously. Compression is achieved at the hardware level via a projection operator that is specifically tailored to the desired machine learning task. The so compressed data are then fed into a specially trained deep network that performs the intended task.}
\label{fig:compdl}
\end{figure}

While our approach is applicable to a wide range of machine learning tasks, in this paper we will mainly focus on image classification. We emphasize that the projection/compression we carry out at the data acquisiton step is completely different from standard (jpeg-type) image compression. Firstly, standard impage compression happens at the software layer and is completely independend from the data acquisition step, while our proposed compression-while-sensing scheme is an intrinsic part of the data acquisition step. Secondly, standard image compression is designed to work for a vast range of images and independent of the task we later perform, while our compression scheme is inherently tied to the image classification task we intend to carry out.


\subsection{Compressive sensing and beyond}

As mentioned above, our approach is in part inspired by ideas from compressive sensing.
The compressive sensing paradigm uses simultaneous sensing and compression. At the core of compressed sensing lies the discovery that it is possible to reconstruct a sparse signal exactly from an underdetermined linear system of equations and that this can be done in a
computationally efficient manner via $\ell_1$-minimization, cf.~\cite{FR13}. Compressive sensing consists of two parts: (i)~The sensing step, which simultaneously compresses the signal. This step is usually implemented (mostly) in hardware. (ii)~The signal reconstruction step via carefully designed algorithms (thus this is done by software).

Compressed sensing aligns with a  few of our objectives, however, they also differ in the following crucial ways:

\begin{enumerate}
\item
Adaptivity: Standard compressive sensing is not adaptive. It considers all possible sparse signals under certain representations. For different data sets, the fundamental assumptions are the same. This assumption is likely too weak for a specific image data set, where only images with certain characteristics are included. For image classification, this is usually the case. Statistical information in the data set may be exploited to achieve better results.

\item
Exact reconstruction: Compressive sensing aims for exact signal reconstruction. That means enough measurements must be taken to ensure all information needed for exact recovery. Clearly, this is too stringent a constraint for image classification where only label recovery is required instead of full signal reconstruction.

\item
Storage and processing cost: It is cumbersome to implement random matrices often proposed in compressive sensing in hardware; only certain structured random matrices can be implemented efficiently.
\end{enumerate}

\subsection{Prior work}
Various approaches have been proposed for improving the computational cost of neural networks and for on-device deep learning.  We point out that in the approaches described below, the assumptions made about the properties of these devices and their capabilities may differ for different approaches.

\begin{enumerate}
\item
Redesigning architecture: Some works proposed new architectures to alleviate the computational burden. This includes MobileNet, LCCL, SqueezeNet, etc.\ (\cite{iandola2016squeezenet, howard2017mobilenets, dong2017more}).

\item
Quantization: It is observed that representing weights in neural networks with lower precision does not lead to serious performance drop in testing accuracy. In this way, the computational cost of basic arithmetic operations is reduced while the model's capability of predicting correct labels is largely preserved.

\item
Matrix decomposition: In terms of runtime, convolutional layers are usually the most costly component of a convolutional neural network. However, many of the weights in these layers are redundant and a large amount of insignificant calculations can be avoided. Instead of quantizing the weights, methods from this category exploit the sparsity of the trained convolutional layers and seek to find some low rank representations of these layers.

\item
Pruning: To exploit the sparsity in neural networks, besides treating each layer as a whole and find a global replacement, it is also possible to deal with each node or connection locally. For a pre-trained neural network, connections with low importance can be removed, reducing the total number of weights. More advanced pruning techniques have been developed, see e.g.~\cite{han2015deep,louizos2017bayesian,aghasi2017net}.

\item Hardware:  Several special purpose processors have been designed for on-device machine learning, see e.g.~\cite{qualcomm2017,intel2018,google2018}.
Recently, an all-optical machine learning device has been built, see~\cite{lin2018all}. The device can perform image classification at the speed of light without needing external power, and it achieves up to 93.39\% classification accuracy for MNIST.
Another work describes a hardware implementation of the kernel method, which is potentially a universal pre-processor for any data~(\cite{saade2016random}. An image is first encoded into a beam of monochromatic laser, and then projected on a random medium. Because of frequent scattering with the nanoparticles in the medium, the transmitted light can be seen as a signal obtained by applying a kernel operation with a random matrix involved. The resulting signal is then collected by a camera and can be fed into a machine learning model. No energy is consumed except for the incoming laser signal during the process.

\end{enumerate}

Our proposed approach is different from all the methods listed above, but can be combined with any one of them. 

Other approaches towards compressing the input in a specific manner before clustering or classificiation can be found e.g.\ in~\cite{davenport2007smashed,hunter2010compressive,ruta2012compressive,reboredo2013compressive,tremblay2016compressive,mcwhirter2018squeezefit}.
However, except for~\cite{mcwhirter2018squeezefit}, these papers do not taylor the compression operator to the classification task, but are rather using compression operators that are in line with classical compressive sensing.

\section{Compressive Deep Learning}
\label{s:math}

In a nutshell our framework can be summarized as follows. First, we construct a compression matrix in form of a projection operator, determined by the entire training set (or a subset of the training set). This compression matrix will be implemented  in hardware, that automatically feeds the compressed data into the next layer of the system. Following the compression matrix there is a neural network that processes the compressed data and recovers the labels. 

There are many possibilities to construct the projection operator. It is important to keep in mind that ultimately the projection/compression step is supposed to be realized in hardware. Therefore it makes sense to impose some structure on the projection operator to make it more amendable to an efficient hardware implementation. 
Unlike compressive sensing, we will not use a random matrix that samples the image space  essentially uniformly at random, but instead we construct a projector that focuses on the regions of interest, i.e., we concentrate our measurement on those regions of the ambient space, in which the images we aim to classify are located, thereby preserving most information with a small number of measurements. To that  end, principal component decomposition would suffice to construct the projection matrix. However, a typical PCA matrix is unstructured and is thus hard to implement efficiently, easily and at low cost in hardware. 

Hence, we need to impose additional condition on the projection operator to be constructed. There is a range of options, but the most convenient one is arguably to consider projections with convolution structure. Convolutions are ubiquitous in signal- and image processing, they are a main ingredient of many machine learning techniques (deep learning being one of them), and they can be implemented efficiently in hardware~\cite{rabiner1975}. 

We will consider two approaches to construct such a convolution-structured projection:
\begin{enumerate}
\item We try to find among all convolution matrices with orthogonal rows the one that is ``most similar''  to the PCA matrix.  While this {\em matrix nearness problem} is non-convex, we will prove that there is a convex problem ``nearby'' and that this convex problem has a convenient explicit solution. This construction of the projection matrix is independent of the CNN we use for image classification.
\item We construct the convolution projection by jointly optimizing the projection matrix and the CNN used for image classification. We do this by adding a ``zeroth'' convolution layer to our image classification CNN. The weights of this zeroth layer will give us the coefficients of the (nonunitary) convolution projection matrix. Of course, for the actual image classification we later remove this zeroth layer, since the whole point is to implement this layer in hardware. In theory this should yield a projection matrix with superior performance, because  this approach jointly optimizes the projection and the classification. But due to the non-convex nature of this optimization problem, there is no guarantee that we can actually find the optimal solution.
\end{enumerate}

\subsection{Construction of projection with convolution structure via constrained matrix nearness}
\label{ss:circapprox}

Given a data set ${\cal D} \subseteq {\mathbb R}^M$, one can find all principal components $\vct{w}_1, \vct{w}_2, \dots, \vct{w}_M$ of this finite point set, which form an orthonormal basis of the space ${\mathbb R}^M$. These $m$ vectors can be grouped into a matrix, with each row being one of the components,
\begin{equation*}
\mtx{P} = 
\begin{bmatrix}
\vct{w}_1^\top, \vct{w}_2^\top, \dots, \vct{w}_M^\top
\end{bmatrix}^\top
\in {\mathbb R}^{M \times M}.
\end{equation*}

The matrix $\mtx{P}$ contains full information about the best approximating affine subspaces of different dimensions to the data set ${\cal D}$. Dimension reduction can be achieved by using the first few rows of $\mtx{P}$. For any $s \leq m$, we can construct
\begin{equation*}
\mtx{P}_s = 
\begin{bmatrix}
\vct{w}_1^\top, \vct{w}_2^\top, \dots, \vct{w}_s^\top
\end{bmatrix}^\top
\in {\mathbb R}^{s \times M},
\end{equation*}
which can serve as a compression matrix that reduces the dimension of the input data from $m$ to $s$. Here, the ratio $\frac{m}{s}$ represents the compression rate of the input.

Yet, in general $\mtx{P}_s$ is not a structured matrix which is required for efficient hardware implementation. As we mentioned before, we try to find a convolution-type matrix that is in some sense as close as possible to  $\mtx{P}_s$. Since we focus on image classification, our input signals are images. Therefore, we consider two-dimensional convolutions. In terms of matrices that means we either deal with block-circulant matrices with circulant blocks (BCCB) or block-Toeplitz matrices with Toeplitz blocks~\cite{davis1979circulant}. Here, for convenience we focus on BCCB matrices. Note however that  as projection, $\mtx{P}_s$ is naturally a fat matrix  (there are more columns than rows), while  BCCB are square matrices by definition. Thus, we are concerned with finding the best approximation of $\mtx{P}_s$ by  subsampled versions of  BCCB matrices. 

Assume samples in the data set ${\cal D}$ are all two-dimensional signals of size $m \times n$; vectorizing these signals yields vectors of length $mn$. The space of subsampled matrices with structure is the space of sub-block-circulant matrices, denoted by $\mathcal{S}_{s, m, n}$, whose elements are matrices with the same dimensions as $\mtx{P}_s$. See the appendix for an exact definition of sub-block-circulant matrices and $\mathcal{S}_{s, m, n}$. We formulate the problem of constructing a convolution-structured projection matrix close to the PCA matrix as  follows: find the sub-block-circulant  matrix $\mtx{C}$, an element of $\mathcal{S}_{s, m, n}$, that is most similar to $\mtx{P}_s$.

How do we measure similarity between matrices? One way to think about this is to identify each matrix with its row space. In our example, the row space of any $\mtx{C} \in \mathcal{S}_{s, m, n}$ is an $s$-dimensional subspace of $\mathbb{R}^{mn}$, the $mn$-dimensional Euclidean space. In other words, we identify each row space with a point in the Grassmannian manifold $G(s, mn)$ and measure the angle $\alpha (\mtx{P}_s, \mtx{C})$ between the row spaces of  $\mtx{P}_s$ and $\mtx{C}$. There are various ways to do so. We follow~\cite{conway1996packing} and use the chordal distance for this purpose. Besides its geometric definition, the chordal distance between the row space of $\mtx{P}_s$ and that of $\mtx{C}$ can be computed via
\begin{equation*}
\alpha (\mtx{P}_s, \mtx{C}) = \frac{1}{2} \| \mtx{P}_s^* \mtx{P}_s - \mtx{C}^* \mtx{C} \|_F.
\end{equation*}

Hence our goal is to find the sub-block-circulant matrix $\mtx{C} \in \mathcal{S}_{s, m, n}$ that minimizes its chordal distance to $\mtx{P}_s$.
\begin{equation}\label{originaloptim}
\min \,\, \vct{\alpha}(\mtx{P}_s, \mtx{C}) \qquad 
\text{s.t.} \,\, \mtx{C} \in \mathcal{S}_{s, m, n}
\end{equation}
This {\em structured matrix nearness problem} does not have an easy solution. Therefore we consider a ``nearby'' problem that does have a  nice explicit solution. We compute
\begin{align}
\begin{split}
\alpha (\mtx{P}_s, \mtx{C}) 
&= \frac{1}{2} \| \mtx{P}_s^* \mtx{P}_s - \mtx{C}^* \mtx{C} \|_F \\
&= \frac{1}{2} \| \mtx{P}_s^* (\mtx{P}_s - \mtx{C}) - (\mtx{P}_s^* - \mtx{C}^*) \mtx{C} \|_F \\
&\leq \frac{1}{2} (\| \mtx{P}_s \|_F + \| \mtx{C} \|_F) \| \mtx{P}_s - \mtx{C} \|_F \\
&= s \| \mtx{P}_s - \mtx{C} \|_F
\end{split}
\end{align}
We used the fact that both $\mtx{P}_s$ and elements of $\mathcal{S}_{s, m, n}$ have unit vector as rows. The above estimation shows that even though $s \| \mtx{P}_s - \mtx{C} \|_F$ does not determine the optimal chordal distance between two row spaces, it is always an upper bound of the chordal distance. Inspired by this estimation, instead of the original optimization problem~\eqref{originaloptim}, we consider the following problem:
\begin{equation} \label{eq:main}
\min \,\, \| \mtx{P}_s - \mtx{C} \|_F \quad
\text{s.t.} \,\, \mtx{C} \in \mathcal{S}_{s, m, n}
\end{equation}
Note that $s$ here is a constant and is determined by the compression rate.
This optimization problem can be solved within a more general setting, namely in the context of complete commuting family of unitary matrices.

To solve \eqref{eq:main}, we consider the more general problem
\begin{equation} \label{eq:opt}
\min_{\mtx{C} \in \mathcal{A}} \,\, \| \mtx{W} - \mtx{C} \|_F,
\end{equation}
where $\mtx{W}$ is a given matrix of dimension $n \times n$ and $\mathcal{A}$ is a commuting family of unitary matrices. The exact definition of $\mathcal{A}$ is given in the appendix.
More precisely, we assume $\mathcal{A}$ is a complete commuting family of unitary matrices and $\mtx{U}$ is a unitary matrix that diagonalizes elements of $\mathcal{A}$. Note that unitary transformation preserves the Frobenius norm since for any $\mtx{X}$,
\begin{equation*}
\| \mtx{U}^* \mtx{X} \mtx{U} \|^2_F 
= \trace (\mtx{U}^* \mtx{X} \mtx{U} (\mtx{U}^* \mtx{X} \mtx{U})^*)
= \trace (\mtx{X} \mtx{U} \mtx{U}^* \mtx{X}^* \mtx{U} \mtx{U}^*)
= \trace (\mtx{X} \mtx{X}^*)
= \| \mtx{X} \|^2_F
\end{equation*}
As a result, the objective function is then
\begin{equation*}
\| \mtx{W} - \mtx{C} \|_F = \| \mtx{U}^* \mtx{W} \mtx{U} - \mtx{U}^* \mtx{C} \mtx{U} \|_F,
\end{equation*}
which can be decomposed into its diagonal part and its off-diagonal part. Since the off-diagonal part is independent of $\mtx{C}$, in order to minimize $\| \mtx{W} - \mtx{C} \|_F$, we only need to minimize the diagonal part. Let $\diag(\vct{c}) = \mtx{U}^* \mtx{C} \mtx{U}$. By definition, $\vct{c} = \theta_U^{-1} (\mtx{C}) \in T^n$. Hence, minimization of the diagonal part becomes
\begin{equation*}
\min_{\vct{c} \in T^n} \| \diag(\mtx{U}^* \mtx{W} \mtx{U}) - \vct{c} \|,
\end{equation*}
where the norm is simply the Euclidean norm in $\mathbb{C}^n$. The minimizer may not be unique, but one  minimizer $\vct{c}_0$ is given by
\begin{align} \label{eq:normalized-vct}
\begin{cases}
(\vct{c}_o)_i = 1, \quad &y_i = 0, \\
(\vct{c}_o)_i = y_i / |y_i|, \quad &\mbox{otherwise},
\end{cases}
\end{align}
where $y_i = (\mtx{U}^* \mtx{W} \mtx{U})_{ii}$. Applying the parametrization map, we get
\begin{equation} \label{eq:opt-sol}
\mtx{C}_o = \theta_U(\vct{c}_o) = \mtx{U} \diag^{-1}(\vct{c}_o) \mtx{U}^* \in \mathcal{A},
\end{equation}
due to the completeness of $\mathcal{A}$, and $\mtx{C}_o$ is a minimizer of \eqref{eq:opt}. Thus, we have proved the following theorem.
\begin{theorem} \label{thm:sol}
For any complete commuting family of unitary matrices $\mathcal{A}$, the optimization problem \eqref{eq:opt} has a solution given by \eqref{eq:opt-sol}.
\end{theorem}

We can readily show that the family of BCCB matrices, denoted by $\mathcal{N}_{m,n}$, is a special case.
\begin{corollary}\label{cor:bccb}
The collection of all unitary block circulant matrices with circulant blocks $\mathcal{N}_{m,n}$ is a complete commuting family of unitary matrices, and a solution to the optimization problem \eqref{eq:opt} is given by
\begin{equation} \label{eq:opt-sol-bccb}
\mtx{C}_o = \theta_U(\vct{c}_o) = \mtx{F}_{m,n} \diag^{-1}(\vct{c}_o) \mtx{F}_{m,n}^* \in \mathcal{A},
\end{equation}
where $\vct{c}_o$ is determined by \eqref{eq:normalized-vct}.
Moreover, if $\mtx{W}$ in \eqref{eq:opt} is a real matrix, then $\mtx{C}_o$ given by \eqref{eq:opt-sol-bccb} is also a real matrix.
\end{corollary}

See the appendix for the proof.


\begin{definition}[Downsampling Operator]
For any $1 \leq k \leq n$, a downsampling operator is 1-1 map $\psi$ from $\mathbb{Z}_k$ to $\mathbb{Z}_n$, where $\mathbb{Z}_p = \{ 1, 2, \dots, p \}$.
\end{definition}
The downsampling operator $\psi$ chooses $k$ elements from $\mathbb{Z}_n$ in a certain order without replacement. We will use $\psi$ to sample the rows of a matrix from $\mathcal{N}_{m,n}$.

\begin{definition}[Subsampled Unitary BCCB]
The Subsampled Unitary BCCB $\mtx{B}$ formed by downsampling $\mtx{C} \in \mathcal{N}_{m,n}$ via the downsampling operator $\psi: \mathbb{Z}_s \rightarrow \mathbb{Z}_{mn}$ is an $s \times mn$ matrix given by
\begin{equation*}
(\mtx{B})_{i, j} = (\mtx{C})_{\psi (i), j}, \quad \forall i \in \mathbb{Z}_s, j \in \mathbb{Z}_{mn}.
\end{equation*}
We denote $\psi(\mtx{C}) = \mtx{B}$. The collection of all such matrices is denoted $\mathcal{S}_{s, m, n; \psi}$, which we will simply refer to as $\mathcal{S}_{s, m, n}$.
\end{definition}

A subsampled unitary BCCB is formed by certain rows of some unitary block circulant matrix with circulant blocks whose indices are determined by a fixed downsampling operator.


Consider a downsampling operator $\psi: \mathbb{Z}_{s} \rightarrow \mathbb{Z}_{mn}$. Let $\mathcal{A}$ be a complete commuting family of unitary matrices. If we subsample every element of $\mathcal{A}$ using $\psi$, we end up getting a set of subsampled unitary matrices,
\begin{equation*}
\mathcal{B} = \{ \psi(\mtx{C}): \mtx{C} \in \mathcal{A} \}.
\end{equation*}
Now, consider the following optimization problem,
\begin{equation} \label{opt-sub}
\min_{\mtx{B} \in \mathcal{B}} \quad \| \mtx{W} - \mtx{B} \|_F,
\end{equation}
where $\mtx{W}$ is a given matrix with dimension of $s \times mn$. Define the \textit{zero-padding map} $\rho_\psi: \mathbb{C}^{s \times mn} \rightarrow \mathbb{C}^{mn \times mn}$ as  follows:
\begin{align*}
( \rho_\psi (\mtx{W}) )_{i,j} =
\begin{cases}
w_{i,j}, \quad &i \in \image \psi, \\
0, \quad &i \notin \image \psi,
\end{cases}
\end{align*}
for any $i, j \in \mathbb{Z}_{mn}$. Observe that any solution $\mtx{C}_0$ to
\begin{align*}
\min_{\mtx{C} \in \mathcal{A}} \| \rho(\mtx{W}) - \mtx{C} \|_F
\end{align*}
gives a solution $\psi(\mtx{C}_0)$ to \eqref{opt-sub}. This is because for any $\mtx{C} \in \mathcal{A}$,
\begin{align*}
\| \rho(\mtx{W}) - \mtx{C} \|_F^2 
&= \| \mtx{W} - \psi(\mtx{C}) \|^2 + \sum_{i \notin \image \psi} \| \vct{c}_i \|^2 \\
&= \| \mtx{W} - \psi(\mtx{C}) \|^2 + mn - k,
\end{align*}
where $\vct{c}_i$ is the $i$-th row of $\mtx{C}$ and is a unit vector.

\begin{theorem} \label{thm:sol-sub}
If $\mathcal{B}$ is a family of subsampled commuting unitary matrices, then an optimization problem of the form \eqref{opt-sub} can be converted into a problem of commuting family of unitary matrices of the form \eqref{eq:opt} using the zero-padding operator.
\end{theorem}

Downsampled from $\mathcal{N}_{m,n}$ by some downsampling operator $\psi$, the collection of subsampled unitary BCCB $\mathcal{S}_{s, m, n}$ is obviously a special case since $\mathcal{S}_{s, m, n} = \psi (\mathcal{N}_{m,n})$. Therefore, the optimization problem 
\begin{align*}
\begin{split}
\min \quad& \| \mtx{P}_s - \mtx{C} \|_F \\
\text{s.t.} \quad& \mtx{C} \in \mathcal{S}_{s, m, n}
\end{split}
\end{align*}
can be solved by following Theorem~\ref{thm:sol-sub} and Theorem~\ref{thm:sol}.
The procedure of solving this problem is summarized in Algorithm~\ref{alg:main}.

\begin{algorithm} \label{alg:main}
\SetKwInOut{Input}{Input}
\SetKwInOut{Output}{Output}

\Input{Data set $\mathcal{D} \subseteq \mathbb{R}^{m \times n}$, dimension $s$ of the compressed data, downsampling operator $\psi$}
\vspace{2pt}
\Output{Compression matrix with structure $\psi (\mtx{C}_0)$}
\vspace{6pt}

Find the first $s$ PCA-components of $\mathcal{D}$ and use them as rows to form an $s \times mn$ matrix $\mtx{P}_s$ \\
\vspace{4pt}
Form an $mn \times mn$ matrix $\rho_\psi (\mtx{P}_s)$ using the zero-padding operator $\rho_\psi$ \\
\vspace{4pt}
Apply the 2D DFT matrix $\mtx{F}_{m,n}$ and extract the diagonal
\begin{equation*}
\vct{c} = \diag (\mtx{F}_{m,n}^* \rho_\psi (\mtx{P}_s) \mtx{F}_{m,n})
\end{equation*} \\
Form a new vector $\vct{c}_0$ by normalizing components of $\vct{c}$. Replace by $1$ if the entry is $0$ \\
\vspace{4pt}
Form a diagonal matrix and apply the 2D DFT matrix. The result is a unitary block circulant matrix with circulant blocks
\begin{equation*}
\mtx{C}_0 = \mtx{F}_{m,n} \diag^{-1} (\vct{c}_0) \mtx{F}_{m,n}^*.
\end{equation*}\\
Subsample the matrix and get $\psi (\mtx{C}_0)$. Use $\psi (\mtx{C}_0)$ to compress the data.
\caption{Construction of deep learning compression matrix}
\end{algorithm}

\subsection{Construction of convolution projection via joint non-convex optimization}
\label{ss:PNN}


In this section we construct a projection matrix by setting up a joint optimization problem that involves both solving the image classification and the finding the optimal projection matrix. To solve this non-convex optimization problem we employ deep learning. More precisely, we construct a CNN to which we add an additional layer as zeroth layer. This zero-th layer will consist of one 2-D convolution  followed by downsampling. The overall architecture is illustrated in Figure~\ref{fig:pnn}.

\begin{figure}
\centering
\includegraphics[width=0.55\textwidth]{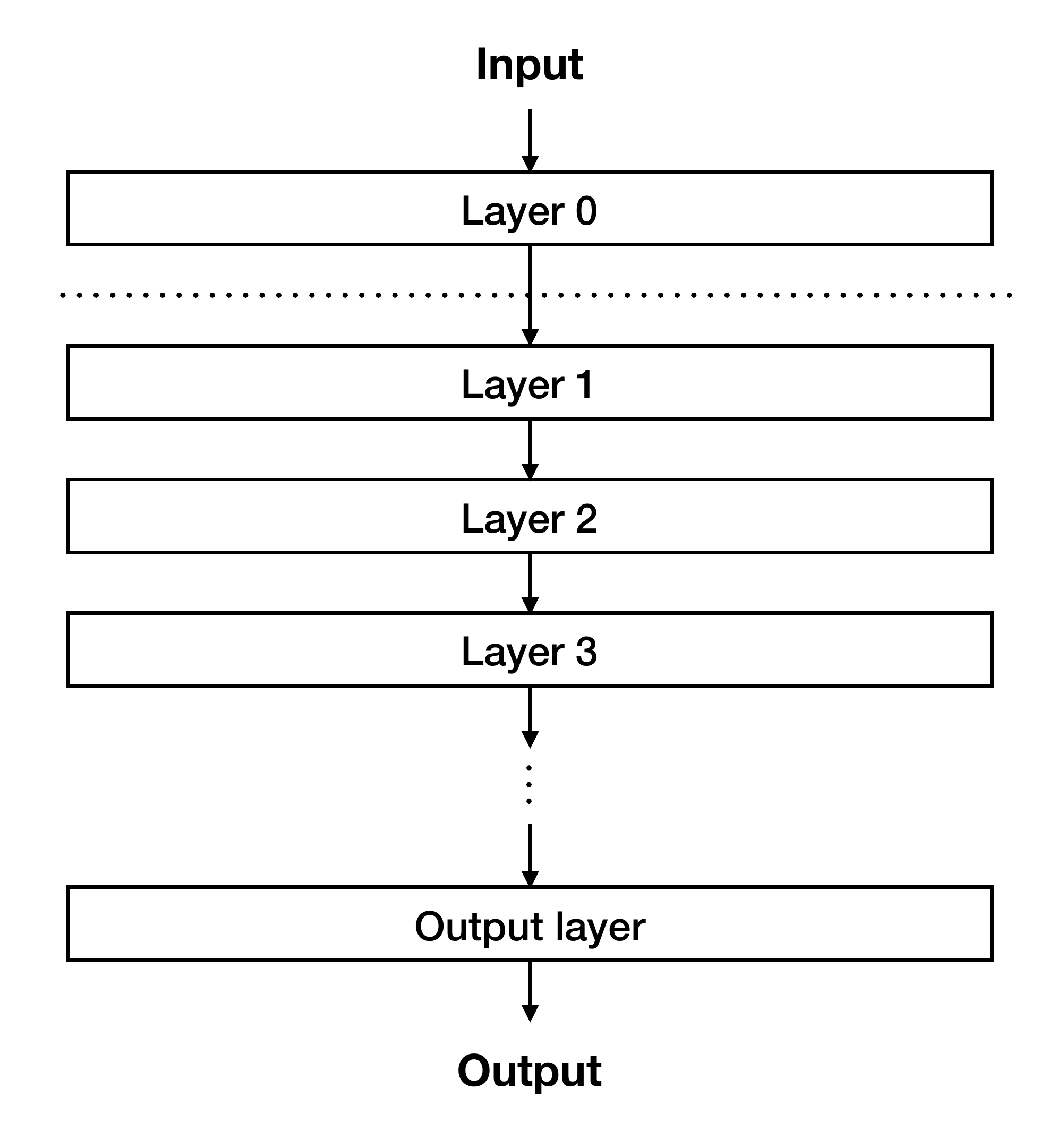}
\caption{The architecture of a CNN with an additional layer. Layers below the zero-th layer form a CNN themselves. The zero-th layer is a convolutional layer with only one filter and a stride greater than one. Thus, the zero-th layer applies a single convolution filter on the input images and then downsamples the results. The output of the zero-th layer is then passed to the next layer. All layers will be trained jointly on a training set.}
\label{fig:pnn}
\end{figure}


\section{Numerical experiments}
To demonstrate the capabilities of our methods, we test both the PNN (name may be changed) and the circulant approximation against a few other baseline methods. We test these methods on two standard test sets,  the MNIST dataset consisting of images depicting  hand-written digits, and the Fashion-MNIST dataset, consisting of  images of fashion products. The general workflow consists of two steps. In the preprocssing step we  subject the images to the projection operator to simulate the compressive image acquistion via hardware. In the second step we feed these images into a convolutional neural network for classification. We will describe the preprocessing methods, the architecture of the network and training details in the rest of the section.

\subsection{Preprocessing step: Projection operator}
In order to put the PNN and the circulant approximation method in context, we conducted a few other methods in the experiments and compare their results against each other. The list below is a brief summary of all these results, and Table~\ref{table:stride} gives the relation between the stride and the compression rate.

\begin{enumerate}
\item
{\bf Downsampling}: Downsample the images using a certain stride. When the stride is equal to $1$, the original dataset is used. Downsampling, i.e.\ just taking low-resolution images, is the easiest (and least sophisticated) way to reduce dimensionality of the images. 
\item
{\bf Random Convolution}: A random convolutional filter of size $5 \times 5$ is generated and then applied to a few locations in the image, determined by the stride. The size of the resulted image is determined by the stride of the convolution applied to the image. The image size is unchanged when the stride is equal to $1$. There is no constraints on the filter size, but we fixed the filter size in our experiments for the sake of simplicity.
Projection matrices of this type have been proposed in the context of compressive sensing (e.g.\ see~\cite{romberg2009compressive}).
\item
{\bf PCA}: Compute the PCA components for the entire training set. Given the compression rate and a raw image, the compressed data is the coefficients of a few leading PCA component for the image. When uncompressed (the compression rate is equal to 1), all coefficients are used.
\item
{\bf Circulant Approximation}: This is the construction outlined in Subsection~\ref{ss:circapprox}. First compute a matrix with structure that is most similar to the PCA matrix of the training set and then compress images using this matrix. The dimension of the matrix with structure is determined by the compression rate. It is a square matrix if the data is uncompressed.
\item
{\bf PNN}:  This is the construction presented in Subsection~\ref{ss:PNN}. Add a convolutional layer with a certain stride and one single feature map on the top of the architecture. Train the resulting network on the original dataset. The first convolutional layer serves as a compressor, and it is optimized along with the rest of the network.
\end{enumerate}

\begin{table}[h]
\begin{center}
\begin{tabular}{c|c|c}
Stride & Dimension & Compression \\ \hline
1        &    $28 \times 28$    &  1.00  \\ \hline
2        &    $14 \times 14$    &  4.00  \\ \hline
3        &    $10 \times 10$    &  7.84  \\ \hline
4        &     $7  \times  7$     &  16.00  \\ \hline
5        &     $6  \times  6$     &  21.78  \\ \hline
6        &     $5  \times  5$     &  31.36
\end{tabular}
\end{center}
\caption{Relation between the stride and the compression rate. The dimension of the raw inputs is $28 \times 28$. After applying one of the preprocessing method with a certain stride, the dimension of the data becomes smaller. The compression is the ratio between the number of pixels in the the raw data and that of the processed data.}
\label{table:stride}
\end{table}

\subsection{Architecture and training}
The neural network for the raw data has the following architecture. The first weighted layer is a convolution layer with $32$  filters of the size $5 \times 5$ and followed by ReLU nonlinearity and a maxpooling layer with stride $2$. The second weighted layer is the same as the first one except that it has $64$ filters, also followed by ReLU and a maxpooling layer with stride $2$. The third weighted layer is a fully connected layer with $256$ units and followed by a dropout layer. The last layer is a softmax layer with $10$ channels, corresponding to the $10$ classes in the dataset. This network has 857,738 weights in total. Architectures of neural networks dealing with compressed data are listed in Table~\ref{tab:architecture}.

The table also lists the total number of weights and the number of floating point operations for each forward pass. In general, both numbers are decreased when we use smaller inputs. Since we are using very small images in the first place, some networks have two pooling layers and some do not have any. This explains why some networks with smaller inputs have more weights and flops than networks with larger inputs. For larger images, we will be able to apply same number of pooling operations both for the raw data and the compressed data. In this way, the decrease in the number of weights and flops will be even more significant.

\begin{table}
\centering
\begin{tabular}{|l|c|c|c|c|c|c|}
\hline
Stride & 1 & 2 & 3 & 4 & 5 & 6 \\ \hline
Input & $28 \times 28$ & $14 \times 14$ & $10 \times 10$ &  $7  \times  7$ & $6  \times  6$ & $5  \times  5$ \\ \hline
First convolution & \multicolumn{6}{|c|}{conv32} \\ \hline
First pooling & \multicolumn{3}{|c|}{maxpool2} &  \multicolumn{3}{|c|}{-} \\ \hline
Second convolution & \multicolumn{6}{|c|}{conv32} \\ \hline
Second pooling & maxpool2 & \multicolumn{5}{|c|}{-} \\ \hline
Fully connected & \multicolumn{6}{|c|}{FC256} \\ \hline
Output & \multicolumn{6}{|c|}{softmax10} \\ \hline
Weights & 857,738 & 857,738 & 464,522 & 857,738 & 644,746 & 464,522 \\ \hline
MegaFlops-sec & 22.93 & 6.94 & 3.54 & 6.70 & 4.92 & 3.42 \\ \hline
\end{tabular}
\caption{Architectures for different input sizes. conv32 means a convolutional layer with 32 filters. All filters have size $5 \times 5$. maxpool2 means a maxpooling layer with stride $2 \times 2$. FC256 represents a fully connected layer with 256 units and softmax10 is the 10-way softmax layer. The second from the last row are the total numbers of weights in these networks.  The last row lists the number of floating point operations for each forward pass. The softmax layer is not included since the cost of the exponential function is hard to estimate. Nonetheless, the last layers of all these networks are the same. Refer to Figure~\ref{fig:results} for the performance of these network architectures.}
\label{tab:architecture}
\end{table}

For PNN, it starts with a convolutional layer with a single filter of the size $5 \times 5$ and a certain stride. The output is then fed into the network described above.

We trained the networks using ADAM. We also did a grid search for hyper-parameters. In general, we found that with learning rate $= 0.01$, dropout rate $=0.4$ generated the best results. Each network was trained for $10$ epochs with mini-batches of the size $32$.

\subsection{Results}
As shown in Figure~\ref{fig:results}, for both datasets, it is evident that both PNN and the circulant approximation achieve higher accuracy rate than downsampling and random convolution, especially when the data are heavily compressed. In most cases, the PNN method works slightly better than the circulant approximation. However, the circulant approximation method exhibits its ability to retain high accuracy when pushing to more extreme compression rate on MNIST. Since PNN is a relaxation of the circulant approximation method, the global optimum of the former is always no worse than the latter. What we observe in MNIST is a result of the training process, which has no guarantee for global optimum of the neural network.

\begin{figure}
\centering
\includegraphics[width=0.95\textwidth]{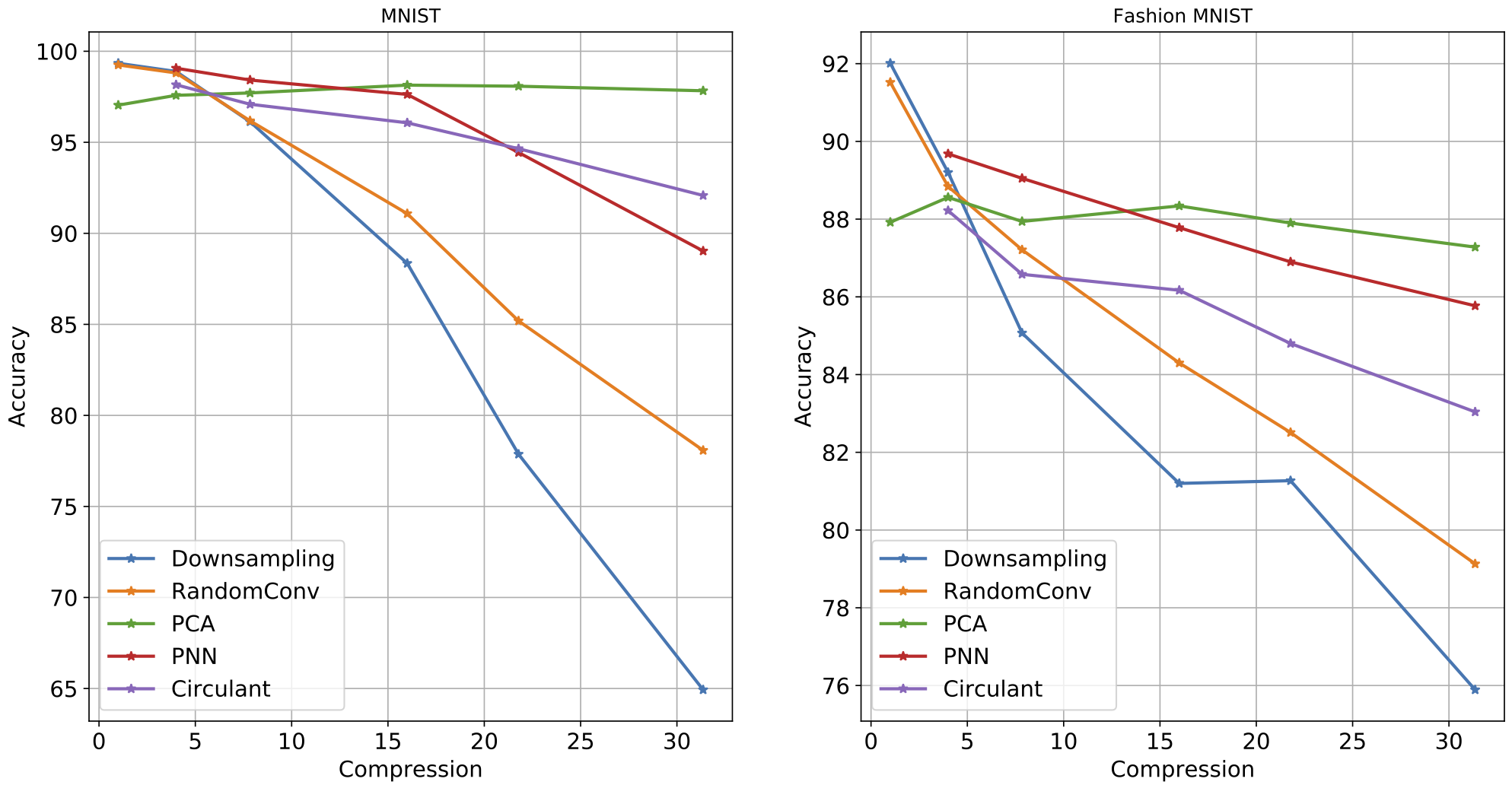}
\caption{Accuracy rates of compressive deep learning with various choices for the projection matrix, tested on MNIST and on Fashion-MNIST. The $x$-axis represents the compression rate of the input (this is determined by the strides). The relation between the strides and the compression is given in Table~\ref{table:stride}.}
\label{fig:results}
\end{figure}

Another interesting result is that the PCA method seems not to be affected by extreme compression rates but rather benefits from them. This is probably because only the coefficients of the leading PCA components have high signal-to-noise ratio, and the rest are mostly noise. Therefore, the PCA method performs better simply by discarding the noisy coefficients. For the purpose of our work, the PCA method cannot be compared with the other methods directly since {\it the PCA matrix is not a convolution and cannot be implemented by hardware}.

\subsection{Reducing the number of filters}
As we have seen in Table~\ref{tab:architecture}, because of the small image size, some networks with smaller inputs have more weights and flops than networks with larger inputs. To further reduce the number of weights and flops for networks with small inputs, we explore the compressibility of these networks by reducing the number of filters. For the network for the case of stride $6$, the network given in Table~\ref{tab:architecture} has 32 filters in the first convolutional layer and 64 filters in the second one. Here, we consider using networks with much less filters. The architectures of these networks are listed below. The network with index $1$ is the same as the network with stride $6$ in Table~\ref{tab:architecture}.

\begin{table}
\centering
\begin{tabular}{|l|c|c|c|c|c|}
\hline
Index & 1 & 2 & 3 & 4 & 5 \\ \hline
Input & \multicolumn{5}{|c|}{$5  \times  5$} \\ \hline
First convolution & conv32 & conv16 & conv8 & conv4 & conv2 \\ \hline
Second convolution & conv64 & conv32 & conv16 & conv8 & conv4 \\ \hline
Fully connected & \multicolumn{5}{|c|}{FC256} \\ \hline
Output & \multicolumn{5}{|c|}{softmax10} \\ \hline
Weights & 464,522 & 220,874 & 108,650 & 54,938 & 28,682 \\ \hline
MegaFlops-sec & 3.42 & 1.07 & 0.37 & 0.15 & 0.06 \\ \hline
\end{tabular}
\caption{Neural networks architectures with fewer filters in their convolutional layers. The number of weights and consequently also the number of flops are reduced.}
\label{tab:smaller}
\end{table}

\begin{figure}
\centering
\includegraphics[width=0.55\textwidth]{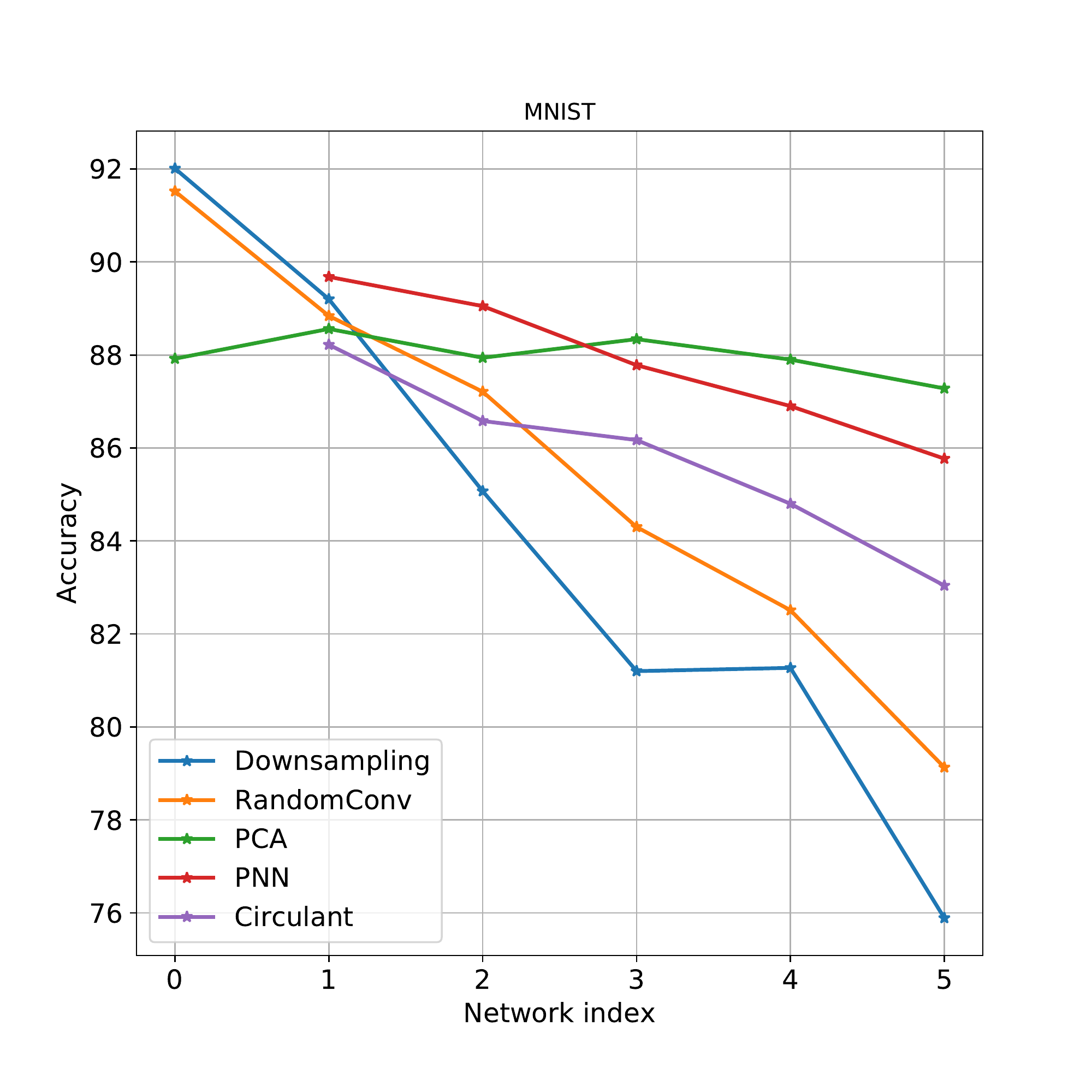}
\caption{Accuracy rates of compressive deep learning with various choices for the projection matrix, tested on MNIST. The $x$-axis represents different network architectures defined in Table~\ref{tab:smaller}. Networks with larger indices have less weights and smaller number of flops.}
\label{fig:smaller}
\end{figure}

The results of these neural networks for MNIST are shown in Figure~\ref{fig:smaller}. Both the PNN and the circulant approximation have relatively high accuracy rates even with very small number of filters.

\section{Appendix}


\subsection{Matrices with circulant strcuture}
Circulant matrices can be thought as discrete convolutions applicable to one dimensional signals. They have the nice property that circulant matrices can be diagonalized by the Discrete Fourier Matrix. For two-dimensional signals, the corresponding matrix is a  block circulant matrix with circulant blocks (BCCB).
\begin{definition}[Block Circulant Matrix with Circulant Blocks]
An $mn \times mn$ block circulant matrix with circulant blocks $\mtx{C}$ has an $m \times m$ circulant block structure and each of its block $\mtx{C}_i \in \mathcal{A}_n$ is an $n \times n$ circulant matrix itself. The matrix $\mtx{C}$ can be written as
\begin{equation*}
\def\arraystretch{1.2}
\mtx{C} =
\begin{bmatrix}
&\mtx{C}_0 & \mtx{C}_{m-1} & \cdots & \mtx{C}_2 & \mtx{C}_1 &\\
&\mtx{C}_1 & \mtx{C}_0 & \mtx{C}_{m-1} &        & \mtx{C}_2 &\\
&\vdots & \mtx{C}_1 & \mtx{C}_0 & \ddots & \vdots &\\
&\mtx{C}_{m-2} &  & \ddots & \ddots & \mtx{C}_{m-1} &\\
&\mtx{C}_{m-1} & \mtx{C}_{m-2} & \cdots & \mtx{C}_1 & \mtx{C}_0 &
\end{bmatrix}.
\end{equation*}
\end{definition}

The lemma below from \cite{davis1979circulant} (Theorem~5.8.1) shows that BCCBs can be diagonalized by the 2D DFT.
\begin{lemma}[Diagonalization by 2D DFT] \label{lem:2d-dft}
Let $\mtx{F}_{m,n}$ be the two-dimensional unitary Discrete Fourier Transform matrix. Then, a matrix $\mtx{C}$ is a block circulant matrix with circulant blocks if and only if $\mtx{F}_{m,n}^* \mtx{C} \mtx{F}_{m,n}$ is a diagonal matrix. In particular, if $\mtx{C}$ is a block circulant matrix with circulant blocks and $\vct{c}$ is its first column, then $\mtx{C}$ can be diagonalized by $\mtx{F}_{m,n}$ as
\begin{equation*}
\mtx{F}_{m,n}^* \mtx{C} \mtx{F}_{m,n} = \sqrt{mn} \diag(\mtx{F}_{m,n} \vct{c}).
\end{equation*}
\end{lemma}


\subsection{Complete commuting family of unitary matrices}

\begin{definition}[Commuting Family of Unitary Matrices]
$\mathcal{A}$ is called a commuting family of unitary matrices if $\mathcal{A}$ is consisted of unitary matrices of the same dimension that commute with each other.
\end{definition}

Properties of commuting families of unitary matrices are presented below. 

\begin{definition}[Parametrization Map]
Let $\mtx{U}$ be a unitary matrix, then the parametrization map $\theta_U$ induced by $\mtx{U}$ is defined by
\begin{equation*}
\theta_U: \mathbb{C}^n \rightarrow \mathbb{C}^{n \times n}, \vct{c} \rightarrow \mtx{U} \diag(\vct{c}) \mtx{U}^*.
\end{equation*}
\end{definition}

\begin{lemma}
Let $\mathcal{A}$ be a commuting family of unitary matrices. There exists a unitary matrix $\mtx{U}$ that diagonalizes any $\mtx{C} \in \mathcal{A}$. Furthermore, the preimage of $\mathcal{A}$ under $\theta_U$ is a subset of the $n$-dimensional torus $T^n$, where $T^n = S^1 \times S^1 \times \dots \times S^1 \subseteq \mathbb{C}^n$.
\end{lemma}
\begin{proof}
A commuting family of matrices can be simultaneously diagonalized, and unitary matrices are normal matrices and hence can be diagonalized by some unitary matrix. Therefore, there exists a unitary matrix $\mtx{U}$ that diagonalizes any $\mtx{C} \in \mathcal{A}$. For any $\mtx{C} \in \mathcal{A}$, $\mtx{U}^* \mtx{C} \mtx{U} = \diag(\vct{c})$ for some $\vct{c} \in \mathbb{C}^n$. Note that entries of $\vct{c}$ are eigenvalues of $\mtx{C}$ since $\mtx{U}$ is unitary. Hence, $\vct{c} \in T^n$ since $\mtx{C}$ is a unitary matrix itself and eigenvalues of a unitary matrix are complex numbers with magnitude $1$. Finally, $\theta_U$ is a one-to-one map and $\theta_U^{-1}(\mtx{C}) = \vct{c}$ since $\mtx{C} = \mtx{U} \diag(\vct{c}) \mtx{U}^*$.
\end{proof}

To ensure solvability of \eqref{eq:opt}, we will need completeness for $\mathcal{A}$.

\begin{definition}[Complete Commuting Family of Unitary Matrices]
A commuting family of unitary matrices $\mathcal{A}$ is called complete if its preimage under the parametrization map is the entire torus. In other words, $\theta_U^{-1}(\mathcal{A}) = T^n$.
\end{definition}

\begin{proof}[Proof of Corollary~\ref{cor:bccb}]
For any $\mtx{C} \in \mathcal{N}_{m,n}$, $\mtx{F}_{m,n}^* \mtx{C} \mtx{F}_{m,n}$ is a diagonal matrix. This shows $\mathcal{N}_{m,n}$ is a commuting family, and the unitary matrix $\mtx{U}$ that diagonalizes all elements of the family is $\mtx{F}_{m,n}$. By its definition, $\mathcal{N}_{m,n}$  consists of unitary matrices. To prove it is complete, we need to show $\theta_{F_{m,n}} (\vct{c}) \in \mathcal{N}_{m,n}$, for any $\vct{c} \in T^n$. Denote $\mtx{C}_0 = \theta_{F_{m,n}} (\vct{c}) = \mtx{F}_{m,n} \diag(\vct{c}) \mtx{F}_{m,n}^*$. On the one hand, thanks to Lemma~\ref{lem:2d-dft}, $\mtx{C}_0 \in \mathcal{BC}_{m,n}$ since $\mtx{F}_{m,n}^* \mtx{C}_0 \mtx{F}_{m,n}$ is diagonal. On the other hand, $\mtx{C}_0$ is unitary since
\begin{align*}
\mtx{C}_0 \mtx{C}_0^* 
&= \mtx{F}_{m,n} \diag(\vct{c}) \mtx{F}_{m,n}^* \mtx{F}_{m,n} \diag(\vct{c^*}) \mtx{F}_{m,n}^* \\
&= \mtx{F}_{m,n} \diag(\vct{c} \odot \vct{c^*}) \mtx{F}_{m,n}^* \\
&= \mtx{F}_{m,n} \mtx{F}_{m,n}^* \\
& = \mtx{I}_{mn}.
\end{align*}
Here, $\vct{c} \odot \vct{c^*} = \boldsymbol{1}$ is the entrywise product and we used the fact that $\vct{c} \in T^n$. This shows $\theta_{F_{m,n}} (\vct{c}) = \mtx{C}_0 \in \mathcal{N}_{m,n}$.

To show the last statement of the corollary,  we will use the fact that the eigenvalues of a real circulant matrix are conjugate symmetric, and vice versa (cf.~\cite{davis1979circulant}, p. 72-76.).
For any $\mtx{C} \in \mathcal{N}_{m, n}$, $\mtx{C}$ can be written as $\mtx{C} = \sum_i \mtx{E}_i \otimes \mtx{C}_i$, where $\mtx{C}_i$ are circulant blocks of $\mtx{C}$ and $\mtx{E}_i = \circulant(\vct{e}_i)$. Hence, $\mtx{F}_{m,n}^* \mtx{C} \mtx{F}_{m,n} = \sum_i \mtx{F}_{m}^* \mtx{E}_i \mtx{F}_{m} \otimes \mtx{F}_{n}^* \mtx{C}_i \mtx{F}_{n}$. Since $\mtx{E}_i$ and $\mtx{C}_i$ are real matrices, both $\mtx{F}_{m}^* \mtx{E}_i \mtx{F}_{m}$ and $\mtx{F}_{n}^* \mtx{C}_i \mtx{F}_{n}$ are diagonal matrices with conjugate symmetric diagonals. Then, $\diag(\mtx{F}_{m}^* \mtx{E}_i \mtx{F}_{m})$ and $\diag(\mtx{F}_{n}^* \mtx{C}_i \mtx{F}_{n})$ are normalized as in \eqref{eq:normalized-vct}. The resulting vectors are still conjugate symmetric. Denote these two vectors $\vct{v}^{(1)}_i$ and $\vct{v}^{(2)}_i$, both column vectors. Thus, $\vct{c}_o = \sum_i \vct{v}^{(1)}_i \otimes \vct{v}^{(2)}_i$, and
$
\mtx{C}_o = \mtx{F}_{m, n} \diag^{-1}(\vct{c}_o) \mtx{F}_{m, n}^* = \sum_i \mtx{M}^{(1)}_i \otimes \mtx{M}^{(2)}_i,
$ where
\begin{align*}
\begin{cases}
\mtx{M}^{(1)}_i = \mtx{F}_{m} \diag^{-1} \left(\vct{v}^{(1)}_i\right) \mtx{F}_{m}^* \\
\mtx{M}^{(2)}_i = \mtx{F}_{n} \diag^{-1} \left(\vct{v}^{(2)}_i\right) \mtx{F}_{n}^*
\end{cases}
\end{align*}
Both $\mtx{M}^{(1)}_i$ and $\mtx{M}^{(2)}_i$ are real, and thus $\mtx{C}_o$ is also real.
\end{proof}

\section{Discussion and Conclusion}
We have introduced the first steps towards developing a principled hybrid hardware-software framework that has the potential to significantly reduce the computational complexity and memory requirements of on-device machine learning. At the same time, the proposed framework raises several challenging questions. How much can we compress data so that we can still with high accuracy conduct the desired machine learning task? In certain cases literature offers some answers (see e.g.~\cite{reboredo2013compressive,tremblay2016compressive,mcwhirter2018squeezefit}), but for more realistic scenarios
research is still in its infancy. Moreover, ideas from transfer learning should help in designing efficient deep networks for compressive input. We hope to address some of these questions in our future work.

\subsubsection*{Acknowledgments}
The authors want to thank Donald Pinckney for discussions and initial simulations related to the topic of this paper.
The authors.\ acknowledge partial support from NSF via grant DMS 1620455 and from NGA and NSF via grant DMS 1737943. 

\bibliography{compdl}
\bibliographystyle{alpha}

\end{document}